\documentclass[dvips,preprint]{imsart}
\usepackage{amsthm,amsmath,amssymb,bm}
\usepackage{graphicx}
\usepackage{multirow}

\newtheorem{lemma}{Lemma}
\newtheorem{prop}{Proposition}
\newtheorem{example}{Example}
\usepackage{url}

\startlocaldefs
\endlocaldefs
\allowdisplaybreaks
\begin{document}

\begin{frontmatter}

\title{Clustering for high-dimension, low-sample size data using distance vectors}
\runtitle{Clustering for HDLSS data using distance vectors}


\author{\fnms{Yoshikazu} \snm{Terada}\corref{}\ead[label=e1]{terada@sigmath.es.osaka-u.ac.jp}}
\address{Graduate School of Engineering Science, Osaka University, 1-3 Machikaneyama,\\ Toyonaka, Osaka, Japan\\ \printead{e1}}
\affiliation{Osaka University}

\runauthor{Y. Terada}

\begin{abstract}
In  high-dimension, low-sample size (HDLSS) data, 
it is not always true that closeness of two objects reflects a hidden cluster structure.
We point out the important fact that it is not the closeness, 
but the ``{\it values}" of distance that contain 
information of the cluster structure in high-dimensional space.
Based on this fact, we propose an efficient and simple clustering approach, called distance vector clustering, for HDLSS data.
Under the assumptions given in the work of Hall et al. (2005),
we show the proposed approach provides a true cluster label under milder conditions
when the dimension tends to infinity with the sample size fixed.
The effectiveness of the distance vector clustering approach is illustrated through a numerical experiment and real data analysis.
\end{abstract}



\end{frontmatter}
%
\section{Introduction}\label{section:1}%
%

In various fields, 
discovering hidden homogeneous classes from high-dimension, low-sample size (HDLSS) data is of significant importance.
Many clustering methods for high-dimensional data have been proposed (e.g., Ahn et al., 2013; Liu et al., 2008;Witte and Tibsirani, 2010). 
One prevalent clustering method operates via variable selection (e.g., Witte and Tibsirani, 2010).
Conversely, Ahn et al. (2013) focused on how to measure the distance between clusters and proposed an efficient clustering method.
In this study, as in the case of Ahn et al. (2013), we focus on the means of measuring the distance between clusters.

Hall et al. (2005) prove the significant fact regarding the geometric representation of data points in HDLSS contexts.
Based on this fact, 
the closeness of the Euclidean distance depends on the mean and variance structure, and does not always contain hidden cluster information.
Thus, a classical clustering method does not always work well for high dimensional data.
For more details about
 the asymptotic behaviors of the classical hierarchical method for high-dimensional data, see Borysov et al. (2013).
Thus, we need a distance measure between clusters for HDLSS data that is more appropriate than the Euclidean distance.
The maximal data piling (MDP) distance (Ahn and Marron, 2010) is one possible choice for measuring the difference between clusters.
The MDP distance was proposed in the context of supervised learning, but we can also apply this distance measure to the case of unsupervised learning.
Ahn et al. (2013) proposed a hierarchical clustering method based on the MDP distance, called MDP clustering.
In Ahn et al. (2013) study, under certain conditions, MDP clustering can detect the difference between mean vectors of ``two" clusters
when the dimension tends to infinity with the sample size fixed.
In addition, 
Ahn et al. (2013) showed that we can approximate MDP clustering by a simple algorithm based on singular value decomposition.
These properties have proven to be quite useful for HDLSS data.

However, the sufficient condition for the label consistency of MDP clustering depends on the sample sizes and variances of the two clusters,
while MDP clustering only focuses on the difference between the mean vectors of two clusters.
Moreover, we cannot detect the differences between the variances of clusters.
In HDLSS contexts, 
there is some possibility that the difference between the variances of each cluster contains the cluster information.
In this study, 
we point out the important fact that it is not the closeness, 
but the ``{\it values}" of the Euclidean distance that contain information regarding the cluster structure in high-dimensional space.
Based on this fact, we propose an efficient and simple clustering approach, called distance vector clustering, for HDLSS data.
By the proposed approach, we can detect not only the differences between mean vectors of clusters but also the differences between the variances of clusters.
Moreover, the computational cost of the proposal approach increases linearly with the number of dimensions of data. 
Under the assumption given in the work of Hall et al. (2005),
we show that the proposed approach also gives the true cluster label under milder conditions
when the dimension tends to infinity with the sample size fixed.

This paper is organized as follows. 
In Section $\ref{section:2}$, some notation and preliminaries are described. 
Then, the difficulty of clustering HDLSS data by the usual method is presented, and the sufficient condition for label consistency of MDP clustering is discussed.
In Section $\ref{section:3}$, the main idea of the proposed method is described, and the algorithm of the distance vector clustering is proposed.
In Section $\ref{section:4}$, sufficient conditions for the asymptotic label consistency of the proposed approach are described.
In Sections $\ref{section:5}$ and $\ref{section:6}$, the effectiveness of the proposed approach is illustrated through a numerical experiment and real data analysis, respectively.

%
\section{Preliminaries and difficulty of clustering HDLSS data}\label{section:2}%
%
Let $K$ be the number of clusters, $N$ be the sample size, and $n_k$ be the sample size of the $k$-th cluster $(k=1,\dots,K)$.
That is, $N=\sum_{k=1}^Kn_k$.
$\bm{X}_k^{(p)}=(X_{k1},\dots,X_{kp})^T$ denotes the $p$-dimensional random vector for the $k$-th cluster $(k=1,\dots,K)$.
For $1\le k \le K$, let the independent and identically distributed (i.i.d.) sample points of the $k$-th cluster be denoted by $\bm{X}_{k1},\dots,\bm{X}_{kn_k}$.
As was done by Ahn et al. (2013), we also assume the following conditions in Hall et al. (2005):
\begin{description}
\item[\rm (a)] $p^{-1}\sum_{s=1}^{p}\mathbb{E}[X_{ks}]^2\rightarrow \mu_{k}^2\quad$ as $\quad p\rightarrow \infty$,
\item[\rm (b)] $p^{-1}\sum_{s=1}^{p}\mathrm{Var}[X_{ks}]^2\rightarrow \sigma_{k}^2\quad$ as $\quad p\rightarrow \infty$,
\item[\rm (c)] $p^{-1}\sum_{s=1}^{p}\{ \mathbb{E}[X_{ks}]^2 -\mathbb{E}[X_{ls}]^2 \}\rightarrow \delta_{kl}^2\quad$ as $\quad p\rightarrow \infty$,
\item[\rm (d)] There exists a permutation of variables, which is $\rho$-mixing for functions that are dominated by quadratics.
\end{description}
Moreover,
let $\eta_{kl}:=\lim_{p\rightarrow \infty}p^{-1}\sum_{s=1}^p \mathbb{E}[X_{ks}]\mathbb{E}[X_{ls}]$.
Under these assumptions, 
Hall et al. (2005) provides the following important facts.
\begin{prop}\label{prop:1} {\rm (Hall et al., 2005)}
Let $\bm{U}_1^{(p)}$ and $\bm{U}_2^{(p)}$ be sample points independently drawn from the distribution of $\bm{X}_1^{(p)}$.
As $p$ goes to infinity,
\begin{align*}
\text{\rm i) }\;  & \frac{1}{\sqrt{p}}\|\bm{X}_1^{(p)}\| \stackrel{\mathbb{P}}{\longrightarrow}\sqrt{\mu_1^2+\sigma_1^2}, & 
\text{\rm iv) }\; & \frac{1}{\sqrt{p}}\|\bm{X}_1^{(p)}-\bm{X}_2^{(p)}\| \stackrel{\mathbb{P}}{\longrightarrow}\sqrt{\delta_{12}^2+\sigma_1^2+\sigma_2^2},\\
\text{\rm ii) }\; & \frac{1}{\sqrt{p}}\|\bm{U}_1^{(p)}-\bm{U}_2^{(p)}\| \stackrel{\mathbb{P}}{\longrightarrow}\sqrt{2}\sigma_1,&
\text{\rm v) }\; &  \frac{1}{p}\langle \bm{X}_1^{(p)},\bm{X}_2^{(p)}\rangle \stackrel{\mathbb{P}}{\longrightarrow} \eta_{12}.\\
\text{\rm iii) }\; &  \frac{1}{p}\langle \bm{U}_1^{(p)},\bm{U}_2^{(p)}\rangle \stackrel{\mathbb{P}}{\longrightarrow} \mu_{1}^2,&
\end{align*}
where $\|\cdot\|$ and $\langle \cdot \rangle$ are the Euclidean norm and the inner product, respectively.
\end{prop}

Based on this fact,
we obtain the sufficient condition for the label consistency of the classical hierarchical clustering method.
\begin{example} 
Let $\bm{U}_1^{(p)},\;\bm{U}_2^{(p)}$ and $\bm{U}_3^{(p)}$ be independent random vectors with the distribution of $\bm{X}_1^{(p)}$.
Let $\bm{V}^{(p)}$ be a random vector with the distribution of $\bm{X}_2^{(p)}$.
Here, we assume that $\bm{U}_1^{(p)},\;\bm{U}_2^{(p)},\;\bm{U}_3^{(p)}$ and $\bm{V}^{(p)}$ are mutually independent.
%
\begin{figure}\label{fig:1}
\begin{minipage}{0.47\textwidth}
\begin{center}
\includegraphics[scale=0.4]{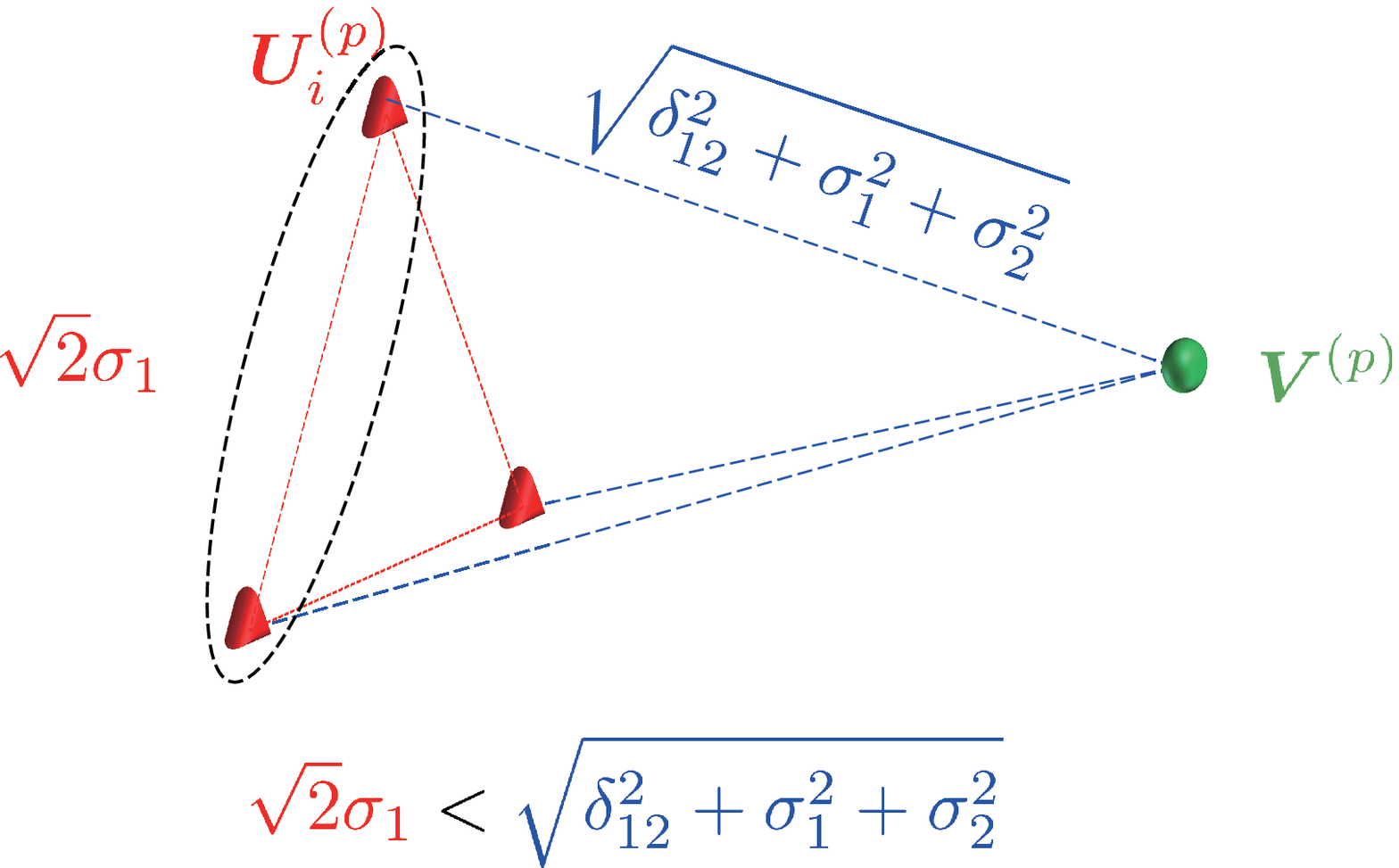}
\par{(a)}
\end{center}
\end{minipage}
\hfill
\begin{minipage}{0.47\textwidth}
\begin{center}
\includegraphics[scale=0.4]{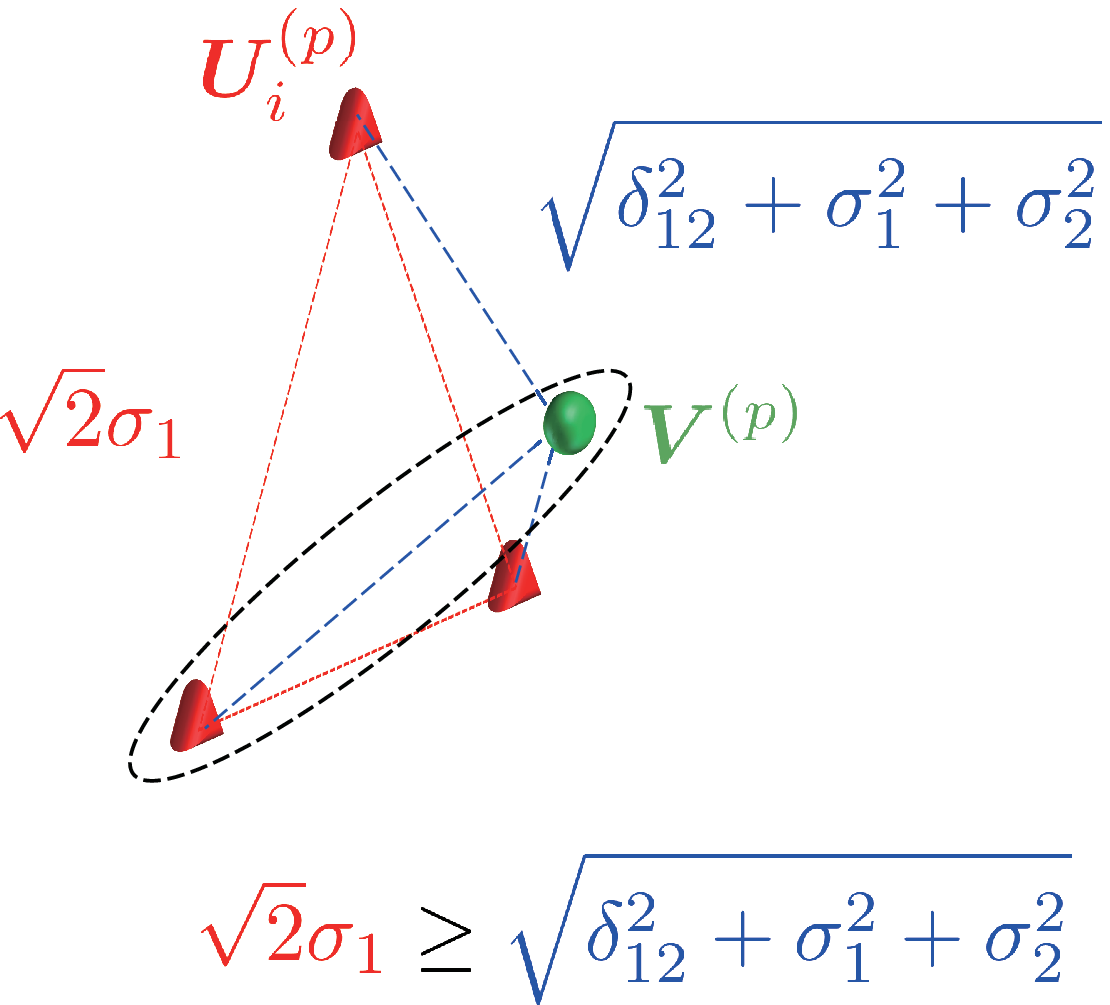}
\par{(b)}
\end{center}
\end{minipage}
\caption{Geometrical representations of $\bm{U}_1^{(p)},\;\bm{U}_2^{(p)},\;\bm{U}_3^{(p)}$ and $\bm{V}^{(p)}$ in the HDLSS contexts.}
\label{data}
\end{figure}
%
Figure $\ref{fig:1}$ shows geometric representations of these objects in the HDLSS context.
If $\sqrt{2}\sigma_1 < \sqrt{\delta_{12}^{2}+\sigma_1^2+\sigma_2^2}$, then
two objects in the same cluster, $\bm{U}_i^{(p)}$ and $\bm{U}_j^{(p)}\;(i\neq j)$, 
may be combined first in the classical hierarchical clustering.
On the other hand, 
if $\sqrt{2}\sigma_1 \ge \sqrt{\delta_{12}^{2}+\sigma_1^2+\sigma_2^2}$, then
two objects in different clusters, $\bm{U}_i^{(p)}$ and $\bm{V}^{(p)}\;(i\neq j)$, 
 may be combined first in classical hierarchical clustering.
Thus, the sufficient condition for the label consistency of classical hierarchical clustering is given by 
$$\sqrt{2}\sigma_1 < \sqrt{\delta_{12}^{2}+\sigma_1^2+\sigma_2^2}.$$
\end{example}

These facts indicate that in HDLSS contexts, the closeness of two objects may not reflect 
the hidden true cluster structure.
Thus, 
Ahn et al. $(2013)$ proposed a clustering method using the maximal data piling distance, called MDP clustering.
The MDP distance between two clusters is defined as the orthogonal distance between
the affine subspaces generated by the sample points in each cluster.
MDP clustering finds successive binary splits, each of which creates two
clusters in such a way that the MDP distance between them is as large as possible.
In Ahn et al. (2013), 
the sufficient condition for label consistency of MDP clustering is given by
\begin{align}
\delta_{12}^2 + \frac{\sigma_1^2}{n_1} + \frac{\sigma_2^2}{n_2}
>
\max
\left\{
\frac{n_1+G}{n_1G}\sigma_1^2
+
\frac{n_2+G}{n_2G}\sigma_2^2
\right\},
\end{align}
where $G\le \min\{n_1,n_2\}$.
Since it is difficult to understand this condition directly,
we consider two specific cases under the conditions $n_1=n_2\;(=:n)$.
First, we consider the case where there is no clear difference between two mean vectors, that is, $\delta_{12}=0$.
In this case, 
the sufficient condition is given by
$$
\frac{\sigma_1^2}{n} + \frac{\sigma_2^2}{n}
>\frac{n+G}{nG}\max\{\sigma_1^2,\sigma_2^2\},
$$
but we have $\sigma_1^2 + \sigma_2^2\le \{(n+G)/G\}\max\{\sigma_1^2,\sigma_2^2\}$ for all $n\in \mathbb{N}$ and the sufficient condition cannot hold.
Thus, we cannot detect the difference between variances of two clusters by MDP clustering.
Next, we consider the case that there is a clear difference between two mean vectors, that is, $\delta_{12}>0$.
In this case, 
the sufficient condition is given by
$$
\delta_{12}^2+\frac{\sigma_1^2}{n} + \frac{\sigma_2^2}{n}
>\frac{n+G}{nG}\max\{\sigma_1^2,\sigma_2^2\}.
$$
For simplicity, we assume that $\sigma:=\sigma_1=\sigma_2$.
If $n\delta_{12}>\{(n-G)/nG\}\sigma^2$, the sufficient condition holds.
Thus, when the difference between mean vectors of two clusters is sufficiently large,
we can discover the true cluster structure.
Consequently, MDP clustering focuses on the difference between the mean vectors of two clusters. 

%
\section{Distance vector clustering}\label{section:3}%
%
\subsection{Main idea and algorithm of distance vector clustering}
In HDLSS data, 
there is some possibility that the differences between variances of each cluster contain the cluster information.
Moreover, 
the sufficient condition of MDP clustering depends on the variances and the sample sizes of two clusters,
whereas MDP clustering focuses on the difference between the mean vectors of two clusters.

In this section, 
we propose a simple and efficient clustering approach based on the usual distance (or inner product) matrix.
Here, 
we first describe the main idea of our approach.
\begin{example}\label{example:2}
Let $\bm{X}_1^{(p)}$ be a sample point drawn from the standard $p$-dimensional normal distribution $N_p(\bm{0},I_p)$.
For fixed $c\neq1$, let $\bm{X}_2^{(p)}$ be a sample point drawn from $N_p(\bm{0},cI_p)$.
Let $\bm{U}_i\;(i=1,\cdots,10)$ be i.i.d. copies of  $\bm{X}_1^{(p)}$ and 
$\bm{V}_i\;(i=1,\cdots,10)$ be i.i.d. copies of  $\bm{X}_2^{(p)}$.
Write $X:=(\bm{U}_1,\dots,\bm{U}_{10},\bm{V}_1,\dots,\bm{V}_{10})^T$.
In this setting, 
the condition of MDP clustering does not hold.
We compute the distance matrix for the data matrix $X$.
Figure $\ref{fig:2}$ shows heatmaps of the distance matrices for various numbers of dimensions.
From this figure, 
we can see that the contrast of the distance matrix between two clusters becomes apparent 
with increasing number of dimensions.
%
\begin{figure}[h]
\label{fig:2}
\begin{center}
\includegraphics[scale=0.185]{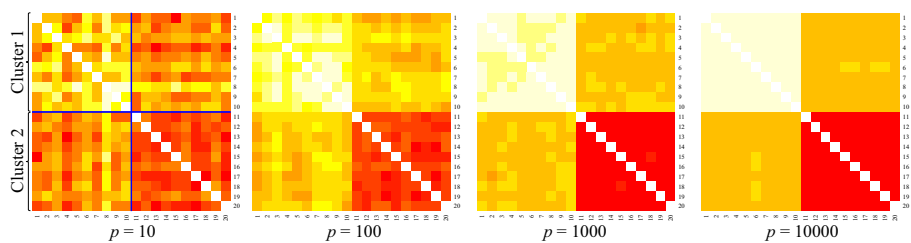}
\end{center}
\caption{Heatmaps of the distance matrices for various number of dimensions ($p=10,\;100\;1000,\;10000$).}
\end{figure}
%
\end{example}

Example $\ref{example:2}$ indicates that 
{\it 
 in HDLSS contexts the closeness between data points may not be meaningful, 
   but “values” of distance contain the true cluster information.
}
Based on this fact, 
we propose the following clustering algorithm:
\begin{description}
\item[Step $1$. ] Compute the usual Euclidean distance matrix $D:=(d_{ij}^{(p)})_{N\times N}$ (or the inner product matrix $S:=XX^T$) from the centered data matrix $X:=(x_{is})_{N\times p}$. 
\item[Step $2$. ] Compute the following distance matrix $\Xi:=(\xi_{ij}^{(p)})_{N\times N}$ from the matrix $D$ (or $S$).
$$
\xi_{ij}^{(p)}=\sqrt{\sum_{t\neq i, j} (d_{it}^{(p)}-d_{jt}^{(p)})^2}\quad \left(\text{ or }= \sqrt{\sum_{t\neq i, j} (s_{it}^{(p)}-s_{jt}^{(p)})^2}\right)
$$ 
\item[Step $3$. ] For the matrix $\Xi$, apply a usual clustering method (e.g., Ward's method).
\end{description}
Figure $\ref{fig:3}$ shows the flow of this algorithm using the inner product matrix $S:=XX^T$ and Ward's method.
%
\begin{figure}[h]
\begin{center}
\includegraphics[scale=0.26]{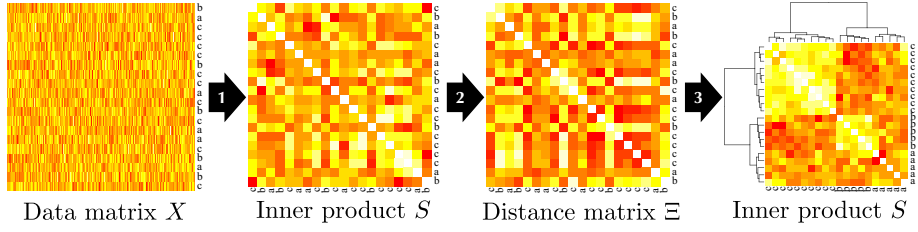}
\end{center}
\caption{Flow of the proposed algorithm using the inner product matrix $S$ and Ward's method.}
\label{fig:3}
\end{figure}
%

\subsection{Theoretical properties}
In this section, we prove the label consistency of this algorithm with a conventional clustering method under the assumption given by Hall et al. (2005).
First, we prove the label consistency for the $k$-means type algorithm.
\subsubsection{$K$-means type}
The objective function of the $k$-means type distance vector clustering method is given by
$$
Q(\mathcal{C}_K\mid K):=
\sum_{i=1}^N\min_k\sum_{j\neq i} (d_{ij}^{(p)}-\bar{d}_{kj}^{(p)})^2
\quad \left(\text{ or }
=\sum_{i=1}^N\min_k\sum_{j\neq i} (s_{ij}^{(p)}-\bar{s}_{kj}^{(p)})^2
\right),
$$
where $\mathcal{C}_K=\{C_1,\dots,C_K\}$ is a partition of objects,
$$
\bar{d}_{kj}^{(p)}=\frac{1}{n_k-1}\sum_{i\neq j, i \in C_k}d_{ij}^{(p)}
\quad \text{and} \quad
\bar{s}_{kj}^{(p)}=\frac{1}{n_k-1}\sum_{i\neq j, i \in C_k}s_{ij}^{(p)}.
$$
We can optimize this function by the usual $k$-means algorithm (e.g., Lloyd's algorithm).

From Proposition $\ref{prop:1}$, we can obtain the following result. 
%
\begin{lemma}\label{lemma:1}
Let $K$ be the true number of clusters.
Under the general assumptions a) - d),
for an arbitrary $K^\ast \ge K$,
$$
\min_{\mathcal{C}_{K^\ast}}Q(\mathcal{C}_{K^\ast}\mid K^\ast)
 \stackrel{\mathbb{P}}{\longrightarrow} 0
 \quad
 \text{as}
 \quad
 p\rightarrow \infty.
$$
\end{lemma}
\begin{proof}
The proof is straightforward.
\end{proof}
Based on Lemma $\ref{lemma:1}$, 
we obtain the sufficient condition for the label consistency of the distance (or inner product) vector clustering approach.
%
\begin{prop}\label{prop:2}
We assume the general assumptions a) - d) and and also assume that $n_k\ge 2\;(k=1,\dots,K)$.
Suppose that the true number of clusters $K$ is given.
\begin{description}
\item[a) ]
If $\forall k,l\;(k\neq l); \sigma_k\neq \sigma_l$ or $\delta_{kl}^2>0$, 
then the estimated cluster label vector with the $k$-means type distance vector clustering method based on the distance matrix $D$
 converges to the true label vector in probability as $p\rightarrow \infty$.
\item[b) ]
Moreover, 
if $\forall k,l\;(k\neq l);\delta_{kl}^2>0$, then 
then the estimated cluster label vector with the $k$-means type distance vector clustering method based on the inner product matrix $S$
 converges to the true label vector in probability as $p\rightarrow \infty$.
\end{description}
\end{prop}
\begin{proof}
Let $\mathcal{C}_K:=\{C_1,\dots,C_K\}$ be the true cluster partition.
From Proposition \ref{prop:1}, for $i,j\in C_k$ we have $\xi_{ij}^{(p)} \stackrel{\mathbb{P}}{\longrightarrow}  0$ as $p\rightarrow \infty$.
For $i,t\in C_k\;(i\neq t)$ and $j\in C_{l}\;(k\neq l)$, as $p\rightarrow \infty$,
$$
d_{it}^{(p)}-d_{jt}^{(p)}\stackrel{\mathbb{P}}{\longrightarrow}\sqrt{2}\sigma_k-\sqrt{\delta_{kl}^2+\sigma_k^2+\sigma_l^2}.
$$
Conversely, for $i\in C_k$ and $j,t\in C_{l}\;(j\neq t;\;k\neq l)$, as $p\rightarrow \infty$,
$$
d_{it}^{(p)}-d_{jt}^{(p)}\stackrel{\mathbb{P}}{\longrightarrow}\sqrt{\delta_{kl}^2+\sigma_k^2+\sigma_l^2}-\sqrt{2}\sigma_l.
$$
For $i \in C_k$ and $j \in C_l\;(k\neq l)$, 
if $\xi_{ij}^{(p)}\stackrel{\mathbb{P}}{\longrightarrow}  0$ as $p\rightarrow \infty$, then we have 
$$
(\sigma_k-\sigma_l)^2=-\delta_{kl}^2, 
$$
which contradicts the assumption $\sigma_k\neq \sigma_l$ or $\delta_{kl}^2>0$.
Thus, we obtain the condition where $\xi_{ij}^{(p)}$ converges in probability to some positive constant for $i \in C_k$ and $j \in C_l\;(k\neq l)$. 
From Lemma $\ref{lemma:1}$ and this fact, we obtain the label consistency of the $k$-means type distance vector clustering method based on the distance matrix $D$.

Next, we consider the $k$-means type clustering based on the inner product matrix $S$.
From Proposition \ref{prop:1}, for $i,j\in C_k$ we also have $\xi_{ij}^{(p)} \stackrel{\mathbb{P}}{\longrightarrow}  0$ as $p\rightarrow \infty$.
For $i,t\in C_k\;(i\neq t)$ and $j\in C_{l}\;(k\neq l)$, as $p\rightarrow \infty$,
$$
s_{it}^{(p)}-s_{jt}^{(p)} \stackrel{\mathbb{P}}{\longrightarrow} \frac{\mu_k^2-\mu_l^2+\delta_{kl}^2}{2}.
$$
Conversely, for $i\in C_k$ and $j,t\in C_{l}\;(j\neq t;\;k\neq l)$, as $p\rightarrow \infty$,
$$
s_{it}^{(p)}-s_{jt}^{(p)} \stackrel{\mathbb{P}}{\longrightarrow} \frac{\mu_k^2-\mu_l^2-\delta_{kl}^2}{2}.
$$
For $i \in C_k$ and $j \in C_l\;(k\neq l)$, 
if $\xi_{ij}^{(p)}\stackrel{\mathbb{P}}{\longrightarrow}  0$ as $p\rightarrow \infty$, then we obtain
$$
\delta_{kl}^2=0, 
$$
which contradicts the assumption $\delta_{kl}^2>0$.
Thus, for the inner product matrix $S$, we also obtain the condition where $\xi_{ij}^{(p)}$ converges in probability to some positive constant for $i \in C_k$ and $j \in C_l\;(k\neq l)$. 
From Lemma $\ref{lemma:1}$ and this fact, we obtain the label consistency of the $k$-means type distance vector clustering method based on the inner product matrix $S$.
\end{proof}
%
\subsubsection{Hierarchical clustering type}
For hierarchical clustering with the matrix $\Xi$, 
the label consistency also holds under the same conditions as that of the $k$-means type method.
The following theorem provides sufficient conditions of the label consistency for the distance vector clustering approach using classical hierarchical clustering methods 
(e.g., the single linkage and Ward's method).
%
\begin{prop}
We assume the general assumptions a) - d) and also assume that  $n_k\ge 2\;(k=1,\dots,K)$.
Let $\mathcal{C}_K:=\{C_1,\dots,C_K\}$ be the true cluster partition.
\begin{description}
\item[a) ]
If $\forall k,l\;(k\neq l); \sigma_k\neq \sigma_l$ or $\delta_{kl}^2>0$, 
then 
$$
\mathbb{P}\left(\forall k,l\;(k\neq l);\max_{i,j\in C_k} \xi_{ij}^{(p)} <  \min_{i\in C_k,j\in C_l}\xi_{ij}^{(p)} \right)
\rightarrow 1
 \quad
 \text{as}
 \quad
 p\rightarrow \infty,
$$
where $\xi_{ij}^{(p)}=\sqrt{\sum_{t\neq i, j} (d_{it}^{(p)}-d_{jt}^{(p)})^2}$.
\item[b) ]
Moreover, 
if $\forall k,l\;(k\neq l);\delta_{kl}^2>0$, then 
$$
\mathbb{P}\left(\forall k,l\;(k\neq l);\max_{i,j\in C_k} \xi_{ij}^{(p)} <  \min_{i\in C_k,j\in C_l}\xi_{ij}^{(p)} \right)
\rightarrow 1
 \quad
 \text{as}
 \quad
 p\rightarrow \infty,
$$
where $\xi_{ij}^{(p)}=\sqrt{\sum_{t\neq i, j} (s_{it}^{(p)}-s_{jt}^{(p)})^2}$.
\end{description}
\end{prop}
\begin{proof}
The proof of this proposition is equivalent to the proof of Proposition $\ref{prop:2}$.
\end{proof}
%

To compare the sufficient condition of MDP clustering, 
we consider the case where the number of clusters is two.
In this case, the sufficient condition of proposed approach using the distance matrix 
is given by 
$$
 \sigma_1\neq \sigma_2 \quad\text{or}\quad \delta_{12}>0.
$$
Moreover, 
the sufficient condition of the proposed approach using the inner product matrix 
is given by 
$$
\delta_{12}>0.
$$
Thus, 
if we use the distance matrix, we can detect the differences between variances or mean vectors.
Alternatively, 
if we use the inner product matrix, we only focus on the differences between mean vectors.
Moreover, 
the sufficient conditions of our approach do not depend on the sample size.
The sufficient condition of our approach using the inner product matrix dose not depend on variances.
In fact, the following example shows that the proposal clusterings with $S$ and $D$ works well, but MDP clustering does not.
\begin{example}\label{example:3}
Let $\bm{\mu}=(0.1,\dots,0.1)^T\in \mathbb{R}^p$ and $\bm{X}_1^{(p)}$ be a sample point drawn from the standard $p$-dimensional normal distribution $N_p(\bm{\mu},I_p)$.
Let $\bm{X}_2^{(p)}$ be a sample point independently drawn from $N_p(-\bm{\mu},1.5\times I_p)$.
Let $\bm{U}_i\;(i=1,\cdots,5)$ be i.i.d. copies of  $\bm{X}_1^{(p)}$ and 
$\bm{V}_i\;(i=1,\cdots,5)$ be i.i.d. copies of  $\bm{X}_2^{(p)}$.
Here, we set $p=2000$.
Write $X:=(\bm{U}_1,\dots,\bm{U}_{5},\bm{V}_1,\dots,\bm{V}_{5})^T$.
In this setting, 
the conditions for the label consistency and the approximation algorithm for MDP clustering does not hold 
while the conditions for the consistency of the proposed approach using $S$ and $D$ hold.
Figure \ref{fig:4} shows the results of these methods. 
Figure \ref{fig:4}, we can see that the proposed approach using $S$ and $D$ works well, but MDP clustering does not.
%
\begin{figure}[!t]
\begin{minipage}{0.3\textwidth}
\begin{center}
\includegraphics[scale=0.65]{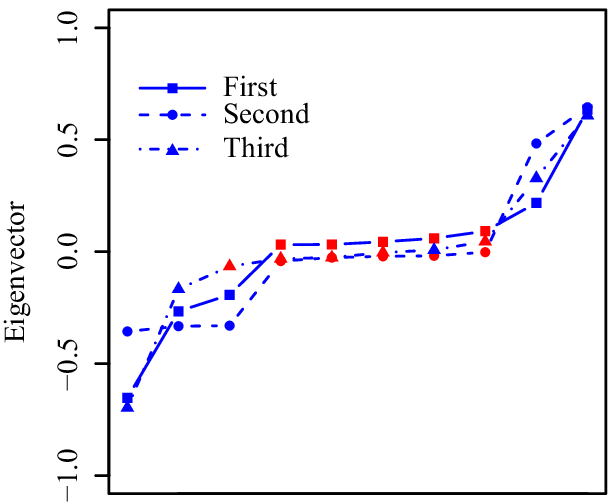}
\par{(a)}
\end{center}
\end{minipage}
\hfill
\begin{minipage}{0.3\textwidth}
\begin{center}
\includegraphics[scale=0.7]{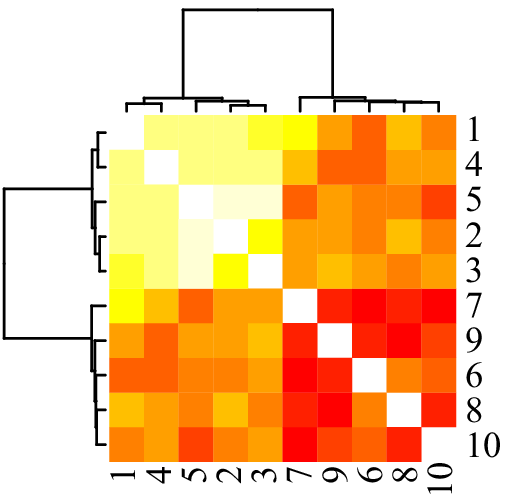}
\par{(b)}
\end{center}
\end{minipage}
\hfill
\begin{minipage}{0.3\textwidth}
\begin{center}
\includegraphics[scale=0.7]{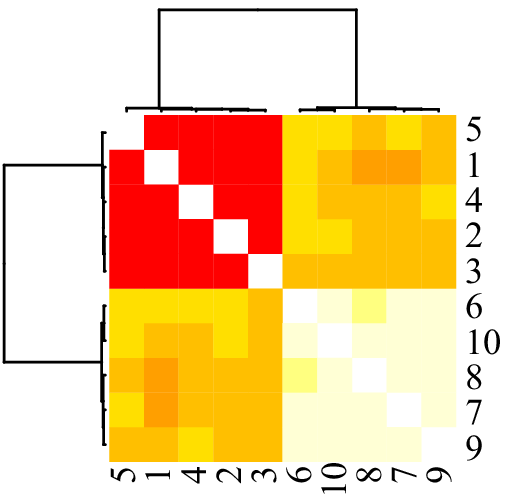}
\par{(c)}
\end{center}
\end{minipage}
\caption{Results of three methods: (a) MDP clustering (the first three sorted eigenvectors of $S$, the colors of the points represent the cluster label)
and the proposed approach using (b) the inner product matrix $S$ and (c) the distance matrix $D$.}
\label{fig:4}
\end{figure}
%
\end{example}

%
\section{Simulation study}\label{section:4}%
%
In this section, 
we illustrate the performance of the proposed approach via numerical experiments.
Here, to compare the proposed approach, 
we choose Ward's method, $k$-means clustering, sparse $k$-means (SK-means) clustering (Witten and Tibshirani, 2010), 
and MDP clustering.
For the proposed approach, we use the Ward type and $k$-means type distance vector clustering methods with the inner product matrix $S$ and the distance matrix $D$. 
Here, we refer to the $k$-means type distance vector clustering methods using $S$ and $D$ as DSKM and DDKM, respectively.
Similarly, we refer to the Ward type distance vector clustering methods using $S$ and $D$ as DSW and DDW, respectively.

In these experiments,
we set the true number of clusters $K=3$, the number of variables $p=1000$ and $2000$, and the sample size $N=100$.
For $\mu>0$, we set 
$$
\bm{\mu}_{1}=(\underbrace{0,\dots,0}_{p}),\;\bm{\mu}_{2}=(\underbrace{\mu,\dots,\mu}_{150},\underbrace{0,\dots,0}_{p-150}),\;
\text{and } \bm{\mu}_{3}=-\bm{\mu}_2
$$
as the centers of the three clusters.
Observations $\bm{X}_i^{(p)}=[X_{i1},\dots,X_{ip}]^T\;(i=1,\dots,N)$ are generated as 
$$
\bm{X}_i^{(p)}:= \sum_{k=1}^3u_{ik}(\bm{\mu}_{k}+\bm{\epsilon}_{ik}),
$$
where $\bm{u}_i=(u_{i1},u_{i2},u_{i3})$ and $\bm{\epsilon}_{ik}\;(i=1,\dots,n)$ are independently generated 
from the multinomial distribution for three trials with probabilities $\bm{\pi}=(\pi_{1},\pi_{2},\pi_{3})$
and the $p$-dimensional normal distribution $N_p(\bm{0}, \Sigma_{k})$, respectively.
In this experiment, we use the following four settings:
\begin{description}
\item{\bf Setting I. } 
Let $\mu=0.6,\;0.8,\;1$ and $\Sigma_k:=I_p\;(k=1,2,3)$, where $I_p$ is the $p\times p$ identity matrix.
Set $\bm{\pi}=(0.7,0.2,0.1)$ and $(1/3,1/3,1/3)$.
\item{\bf Setting II. } 
Let $\mu=0.6,\;0.8,\;1$ and 
$$
\Sigma_k:=
\begin{bmatrix}
\tilde{\Sigma}_{k} & O_{150\times (p-150)} \\
O_{(p-150)\times 150} & I_{(p-150)}
\end{bmatrix}\quad(k=1,2,3),
$$
where $\tilde{\Sigma}_{k}:=(\rho_{k}^{|s-t|/3} )_{150\times 150}$ 
and $\rho_{k}\;(k=1,2,3)$ are generated from the uniform distribution on the interval $[0.1,0.9]$.
Set $\bm{\pi}=(0.7,0.2,0.1)$ and $(1/3,1/3,1/3)$.
\item{\bf Setting III. } 
Let $\mu=0.5$ and $\Sigma_1:=\sigma^2\times I_p,\;\Sigma_2:=I_p$ and $\Sigma_3:= I_p$,
where $\sigma^2=1.5,\;2.0,\;2.5$.
Set $\bm{\pi}=(0.7,0.2,0.1)$ and $(1/3,1/3,1/3)$.
\item{\bf Setting IV. } 
Let $\mu=0.5$ and
$$
\Sigma_k:=
\begin{bmatrix}
\tilde{\Sigma}_{k} & O_{150\times (p-150)} \\
O_{(p-150)\times 150} & I_{(p-150)}
\end{bmatrix}\quad(k=1,2,3),
$$
where $\tilde{\Sigma}_{1}:=(\tilde{\sigma}_{1st} )_{150\times 150}$,
$$
\tilde{\sigma}_{1st} := 
\begin{cases}
\sigma^2	& (s=t)\\
\rho_{1}^{|s-t|/3}	& (s\neq t)
\end{cases}\quad(\sigma^2=1.5,\;2.0,\;2.5),
$$
$\tilde{\Sigma}_{k}:=(\rho_{k}^{|s-t|/3} )_{150\times 150}\;(k=2,3)$ and $\rho_{k}\;(k=1,2,3)$ are generated from the uniform distribution on the interval $[0.1,0.9]$.
Set $\bm{\pi}=(0.7,0.2,0.1)$ and $(1/3,1/3,1/3)$.
\end{description}

We constructed $100$ datasets for each setting and applied the eight methods to each standardized dataset with zero means and unit variances.
To compare the results of the eight clustering methods, we used the adjusted rand index (ARI) (Hubert and Arabie, 1985). 
Note that we applied to MDP clustering with the turning parameters $T=1,\;2,\;3$ and $G=\min\{5,n_1,n_2,n_3\}$,
and employed the best ARI score as the ARI score listed for the MDP clustering results.

Table $\ref{table:SI}$ shows the average ARI scores and their standard errors for each method in Setting I.
In Setting I, there are only differences between mean vectors of the three clusters.
Even if $\mu=0.6$, MDP clustering and the distance vector clustering approach using the inner product matrix work well.
Since DDW and DDKM focus on differences between both mean vectors and variances,
these methods do not work well in this setting.
Table $\ref{table:SII}$ shows the average ARI scores and their standard errors for each method in Setting II.
In Setting II, there are also only differences between mean vectors of the three clusters but informative variables are correlated.
The data for Setting II show a similar tendency to that of Setting I, although the ARI scores of Setting II are overall lower than those of Setting I.
\begin{table}[h]
\caption{Average ARI scores and their standard errors for each method in Setting I.}
\label{table:SI}
\begin{tabular}{c|c|c|c|ccc}
\hline
Setting			& $p$		&$\bm{\pi}$	&method			& $\mu=0.6$	& $\mu=0.8$	& $\mu=1.0$	\\
\hline\hline
\multirow{32}{*}{I}	& \multirow{16}{*}{1000}	&\multirow{8}{*}{\begin{rotatebox}{90}{(0.7,0.2,0.1)}\end{rotatebox}}
									&	Ward		& 0.403 (0.167) & 0.842 (0.151) & 0.974 (0.026) 	\\
				&			&		&	$K$-means	& 0.597 (0.234) & 0.897 (0.204) & 0.999 (0.004)  	\\
				&			&		&	SK-means		& 0.621 (0.278) & 0.930 (0.289) & 1.000 (0.000)   	\\
				&			&		&	MDP			& 0.875 (0.127) & 0.992 (0.026) & 1.000 (0.000)   	\\
				&			&		&	DSW			& 0.935 (0.072) & 0.998 (0.008) & 1.000 (0.000)    	\\
				&			&		&	DSKM		& 0.919 (0.113) & 0.990 (0.066) & 1.000 (0.000)  	\\
				&			&		&	DDW		& 0.259 (0.183) & 0.611 (0.251) & 0.986 (0.078)   	\\
				&			&		&	DDKM		& 0.257 (0.121) & 0.552 (0.159) & 0.940 (0.070)  	\\ \cline{3-7}
				&					&\multirow{8}{*}{\begin{rotatebox}{90}{(1/3,1/3,1/3)}\end{rotatebox}}
									&	Ward		& 0.599 (0.085) & 0.820 (0.069) & 0.943 (0.046)	\\
				&			&		&	$K$-means	& 0.954 (0.035) & 0.998 (0.008) & 1.000 (0.000) 	\\
				&			&		&	SK-means		& 0.987 (0.020) & 0.999 (0.005) & 1.000 (0.000)  	\\
				&			&		&	MDP			& 0.927 (0.102) & 0.997 (0.009) & 1.000 (0.000)  	\\
				&			&		&	DSW			& 0.987 (0.019) & 1.000 (0.000) & 1.000 (0.000)   	\\
				&			&		&	DSKM		& 0.998 (0.008) & 1.000 (0.000) & 1.000 (0.000)  	\\
				&			&		&	DDW		& 0.780 (0.163) & 0.997 (0.011) & 0.999 (0.004)  	\\
				&			&		&	DDKM		& 0.638 (0.108) & 0.964 (0.037) & 0.997 (0.009) 	\\ 
\cline{2-7}
				& \multirow{16}{*}{2000}	&\multirow{8}{*}{\begin{rotatebox}{90}{(0.7,0.2,0.1)}\end{rotatebox}}
									&	Ward		& 0.170 (0.110) & 0.609 (0.211) & 0.852 (0.120) 	\\
				&			&		&	$K$-means	& 0.390 (0.139) & 0.665 (0.276) & 0.953 (0.158)  	\\
				&			&		&	SK-means		& 0.345 (0.230) & 0.845 (0.252) & 0.977 (0.115)   	\\
				&			&		&	MDP			& 0.627 (0.209) & 0.966 (0.070) & 0.997 (0.010)   	\\
				&			&		&	DSW			& 0.677 (0.240) & 0.991 (0.019) & 1.000 (0.000)    	\\
				&			&		&	DSKM		& 0.588 (0.284) & 0.985 (0.060) & 0.988 (0.084)  	\\
				&			&		&	DDW		& 0.072 (0.069) & 0.335 (0.193) & 0.650 (0.271)   	\\
				&			&		&	DDKM		& 0.057 (0.050) & 0.342 (0.131) & 0.613 (0.186)  	\\ \cline{3-7}
				&					&\multirow{8}{*}{\begin{rotatebox}{90}{(1/3,1/3,1/3)}\end{rotatebox}}
									&	Ward		& 0.458 (0.079) & 0.666 (0.091) & 0.820 (0.058)	\\
				&			&		&	$K$-means	& 0.817 (0.099) & 0.976 (0.029) & 0.996 (0.012) 	\\
				&			&		&	SK-means		& 0.989 (0.017) & 0.999 (0.004) & 1.000 (0.000)  	\\
				&			&		&	MDP			& 0.735 (0.150) & 0.956 (0.073) & 0.995 (0.011)  	\\
				&			&		&	DSW			& 0.943 (0.040) & 0.999 (0.006) & 1.000 (0.000)   	\\
				&			&		&	DSKM		& 0.986 (0.019) & 0.999 (0.004) & 1.000 (0.000)  	\\
				&			&		&	DDW		& 0.487 (0.147) & 0.896 (0.124) & 0.993 (0.027)  	\\
				&			&		&	DDKM		& 0.450 (0.072) & 0.761 (0.117) & 0.949 (0.044) 	\\ 
\hline
\end{tabular}
\end{table}

\begin{table}[h]
\caption{Average ARI scores and their standard errors for each method in Setting II.}
\label{table:SII}
\begin{tabular}{c|c|c|c|ccc}
\hline
Setting			& $p$		&$\bm{\pi}$	&method			& $\mu=0.6$	& $\mu=0.8$	& $\mu=1.0$	\\
\hline\hline
\multirow{32}{*}{II}	& \multirow{16}{*}{1000}	&\multirow{8}{*}{\begin{rotatebox}{90}{(0.7,0.2,0.1)}\end{rotatebox}}
									&	Ward		& 0.260 (0.152) & 0.607 (0.171) & 0.807 (0.092) 	\\
				&			&		&	$K$-means	& 0.311 (0.164) & 0.723 (0.200) & 0.898 (0.119)  	\\
				&			&		&	SK-means		& 0.172 (0.145) & 0.667 (0.255) & 0.894 (0.128)   	\\
				&			&		&	MDP			& 0.523 (0.164) & 0.730 (0.177) & 0.902 (0.094)   	\\
				&			&		&	DSW			& 0.496 (0.192) & 0.818 (0.144) & 0.925 (0.093)    	\\
				&			&		&	DSKM		& 0.471 (0.163) & 0.858 (0.108) & 0.933 (0.081)  	\\
				&			&		&	DDW		& 0.202 (0.116) & 0.420 (0.196) & 0.803 (0.178)   	\\
				&			&		&	DDKM		& 0.213 (0.110) & 0.430 (0.153) & 0.786 (0.132)  	\\ \cline{3-7}
				&					&\multirow{8}{*}{\begin{rotatebox}{90}{(1/3,1/3,1/3)}\end{rotatebox}}
									&	Ward		& 0.403 (0.070) & 0.642 (0.091) & 0.824 (0.059)	\\
				&			&		&	$K$-means	& 0.472 (0.093) & 0.860 (0.072) & 0.938 (0.041) 	\\
				&			&		&	SK-means		& 0.467 (0.091) & 0.867 (0.062) & 0.940 (0.041)  	\\
				&			&		&	MDP			& 0.440 (0.131) & 0.779 (0.160) & 0.910 (0.119)  	\\
				&			&		&	DSW			& 0.541 (0.102) & 0.870 (0.078) & 0.961 (0.051)   	\\
				&			&		&	DSKM		& 0.570 (0.094) & 0.905 (0.052) & 0.952 (0.039)  	\\
				&			&		&	DDW		& 0.407 (0.104) & 0.790 (0.127) & 0.914 (0.083)  	\\
				&			&		&	DDKM		& 0.431 (0.087) & 0.788 (0.087) & 0.868 (0.072) 	\\ 
\cline{2-7}
				& \multirow{16}{*}{2000}	&\multirow{8}{*}{\begin{rotatebox}{90}{(0.7,0.2,0.1)}\end{rotatebox}}
									&	Ward		& 0.064 (0.045) & 0.301 (0.149) & 0.442 (0.167) 	\\
				&			&		&	$K$-means	& 0.064 (0.029) & 0.295 (0.123) & 0.483 (0.180)  	\\
				&			&		&	SK-means		& 0.062 (0.034) & 0.157 (0.103) & 0.365 (0.195)   	\\
				&			&		&	MDP			& 0.067 (0.076) & 0.537 (0.156) & 0.538 (0.136)   	\\
				&			&		&	DSW			& 0.099 (0.062) & 0.386 (0.182) & 0.510 (0.187)    	\\
				&			&		&	DSKM		& 0.099 (0.055) & 0.433 (0.138) & 0.546 (0.175)  	\\
				&			&		&	DDW		& 0.126 (0.075) & 0.242 (0.127) & 0.322 (0.188)   	\\
				&			&		&	DDKM		& 0.133 (0.055) & 0.241 (0.108) & 0.387 (0.172)  	\\ \cline{3-7}
				&					&\multirow{8}{*}{\begin{rotatebox}{90}{(1/3,1/3,1/3)}\end{rotatebox}}
									&	Ward		& 0.280 (0.066) & 0.510 (0.081) & 0.601 (0.106)	\\
				&			&		&	$K$-means	& 0.341 (0.081) & 0.727 (0.081) & 0.794 (0.072) 	\\
				&			&		&	SK-means		& 0.298 (0.099) & 0.776 (0.073) & 0.788 (0.065)  	\\
				&			&		&	MDP			& 0.253 (0.118) & 0.581 (0.180) & 0.651 (0.174)  	\\
				&			&		&	DSW			& 0.362 (0.077) & 0.744 (0.095) & 0.788 (0.093)   	\\
				&			&		&	DSKM		& 0.384 (0.076) & 0.812 (0.054) & 0.818 (0.074)  	\\
				&			&		&	DDW		& 0.284 (0.096) & 0.547 (0.140) & 0.684 (0.122)  	\\
				&			&		&	DDKM		& 0.292 (0.069) & 0.540 (0.096) & 0.664 (0.088) 	\\ 
\hline
\end{tabular}
\end{table}

Table $\ref{table:SIII}$ shows the average ARI scores and their standard errors for each method in Setting III.
In Setting III, there are both differences among mean vectors of the three clusters and
the variances of the first cluster having a zero mean vector are larger than those of other clusters.
Thus, it is difficult to detect differences between mean vectors by usual clustering methods.
In fact, $K$-means, SK-means, and MDP clusterings do not work well overall, 
while the distance vector clustering approach using the distance matrix does.
Moreover, the distance vector clustering approach using the inner product matrix appears unaffected by the variances,
while MDP clustering does not work well.
In fact, the ARI scores of DSW with $\sigma^2=1.5,2.0,2.5$ are approximately similar in each case.
From this fact, we can confirm that the sufficient condition of the label consistency for the distance vector clustering approach using the inner product matrix 
does not depend on the variances and the sample size.
Note that for the distance vector clustering approach using the inner product matrix, the Ward type algorithm performs better than  the $k$-means type algorithm.
Conersely, 
for the distance vector clustering approach using the distance matrix, 
 the Ward type algorithm works well for balanced cases ($\bm{\pi}=(1/3,1/3,1/3)$), while the $k$-means type algorithm does not, 
 and 
 the $k$-means type algorithm works well for unbalanced cases ($\bm{\pi}=(0.7,0.2,0.1)$), while the Ward  type algorithm does not.
 In addition, we mention that in this setting,
we clearly obtain the conditions
$$
2\sigma_2^2=2\sigma_3^2 < \delta_{23}^2 + \sigma_2^2 + \sigma_3^2<\delta_{12}^2 + \sigma_1^2 + \sigma_2^2=\delta_{13}^2 + \sigma_1^2 + \sigma_3^2.
$$
Thus, the results of Ward's method are also acceptable in this setting.
For more details about the asymptotic behaviors of Ward's method for high-dimensional data, see Borysov et al. (2013).
Setting IV is similar to Setting III, but informative variables are correlated.
As with Setting I and II,
the data for Setting IV given in Table $\ref{table:SIV}$ show a similar tendency to that of Setting III, although the ARI scores of Setting IV are overall lower than those of Setting III.
\begin{table}[h]
\caption{Average ARI scores and their standard errors for each method in Setting III.}
\label{table:SIII}
\begin{tabular}{c|c|c|c|ccc}
\hline
Setting			& $p$		&$\bm{\pi}$	&method			& $\sigma^2=2.5$	& $\sigma^2=2$	& $\sigma^2=1.5$	\\
\hline\hline
\multirow{32}{*}{III}	& \multirow{16}{*}{1000}	&\multirow{8}{*}{\begin{rotatebox}{90}{(0.7,0.2,0.1)}\end{rotatebox}}
									&	Ward		&  0.518 (0.077) &  0.478 (0.165) & 0.386 (0.166) 	\\
				&			&		&	$K$-means	& -0.096 (0.077) &  0.032 (0.062) & 0.096 (0.048)  	\\
				&			&		&	SK-means		& -0.055 (0.080) &  0.020 (0.057) & 0.080 (0.053)   	\\
				&			&		&	MDP			& -0.122 (0.018) & -0.122 (0.018) & 0.077 (0.128)   	\\
				&			&		&	DSW			&  0.583 (0.267) &  0.491 (0.286) & 0.538 (0.272)    	\\
				&			&		&	DSKM		&  0.072 (0.069) &  0.079 (0.077) & 0.212 (0.137)  	\\
				&			&		&	DDW		&  0.502 (0.065) &  0.486 (0.055) & 0.499 (0.063)   	\\
				&			&		&	DDKM		&  0.746 (0.224) &  0.747 (0.230) & 0.568 (0.191)  	\\ \cline{3-7}
				&					&\multirow{8}{*}{\begin{rotatebox}{90}{(1/3,1/3,1/3)}\end{rotatebox}}
									&	Ward		& 0.992 (0.013) & 0.994 (0.014) & 0.980 (0.040)	\\
				&			&		&	$K$-means	& 0.480 (0.069) & 0.499 (0.078) & 0.556 (0.076) 	\\
				&			&		&	SK-means		& 0.481 (0.068) & 0.506 (0.077) & 0.620 (0.105)  	\\
				&			&		&	MDP			& 0.416 (0.189) & 0.503 (0.085) & 0.531 (0.087)  	\\
				&			&		&	DSW			& 0.988 (0.023) & 0.973 (0.037) & 0.922 (0.045)   	\\
				&			&		&	DSKM		& 0.805 (0.144) & 0.889 (0.069) & 0.945 (0.036)  	\\
				&			&		&	DDW		& 1.000 (0.000) & 1.000 (0.020) & 0.995 (0.020)  	\\
				&			&		&	DDKM		& 0.577 (0.048) & 0.574 (0.188) & 0.868 (0.188) 	\\ 
\cline{2-7}
				& \multirow{16}{*}{2000}	&\multirow{8}{*}{\begin{rotatebox}{90}{(0.7,0.2,0.1)}\end{rotatebox}}
									&	Ward		&  0.488 (0.069) &  0.491 (0.059) &  0.137 (0.228) 	\\
				&			&		&	$K$-means	& -0.147 (0.014) & -0.137 (0.020) &  0.018 (0.075)  	\\
				&			&		&	SK-means		& -0.111 (0.059) & -0.084 (0.062) &  0.017 (0.061)   	\\
				&			&		&	MDP			& -0.126 (0.018) & -0.126 (0.017) & -0.114 (0.026)   	\\
				&			&		&	DSW			&  0.574 (0.214) &  0.416 (0.192) &  0.243 (0.206)    	\\
				&			&		&	DSKM		& -0.050 (0.074) &  0.023 (0.064) &  0.074 (0.059)  	\\
				&			&		&	DDW		&  0.471 (0.053) &  0.496 (0.064) &  0.490 (0.056)   	\\
				&			&		&	DDKM		&  0.779 (0.218) &  0.803 (0.206) &  0.798 (0.209)  	\\ \cline{3-7}
				&					&\multirow{8}{*}{\begin{rotatebox}{90}{(1/3,1/3,1/3)}\end{rotatebox}}
									&	Ward		& 0.463 (0.032) & 0.731 (0.250) & 0.938 (0.068)	\\
				&			&		&	$K$-means	& 0.054 (0.035) & 0.442 (0.136) & 0.510 (0.093) 	\\
				&			&		&	SK-means		& 0.051 (0.036) & 0.438 (0.144) & 0.532 (0.141)  	\\
				&			&		&	MDP			& 0.105 (0.053) & 0.129 (0.169) & 0.478 (0.115)  	\\
				&			&		&	DSW			& 0.998 (0.009) & 0.964 (0.033) & 0.835 (0.069)   	\\
				&			&		&	DSKM		& 0.653 (0.130) & 0.698 (0.128) & 0.823 (0.123)  	\\
				&			&		&	DDW		& 0.814 (0.244) & 0.828 (0.237) & 0.880 (0.184)  	\\
				&			&		&	DDKM		& 0.568 (0.046) & 0.576 (0.048) & 0.567 (0.062) 	\\ 
\hline
\end{tabular}
\end{table}

\begin{table}[h]
\caption{Average ARI scores and their standard errors for each method in Setting IV.}
\label{table:SIV}
\begin{tabular}{c|c|c|c|ccc}
\hline
Setting			& $p$		&$\bm{\pi}$	&method			& $\sigma^2=2.5$	& $\sigma^2=2$	& $\sigma^2=1.5$	\\
\hline\hline
\multirow{32}{*}{IV}	& \multirow{16}{*}{1000}	&\multirow{8}{*}{\begin{rotatebox}{90}{(0.7,0.2,0.1)}\end{rotatebox}}
									&	Ward		&  0.508 (0.077) &  0.477 (0.171) & 0.258 (0.143) 	\\
				&			&		&	$K$-means	& -0.004 (0.081) &  0.036 (0.035) & 0.079 (0.044)  	\\
				&			&		&	SK-means		& -0.001 (0.069) &  0.026 (0.035) & 0.054 (0.043)   	\\
				&			&		&	MDP			& -0.112 (0.042) & -0.079 (0.065) & 0.119 (0.106)   	\\
				&			&		&	DSW			&  0.416 (0.242) &  0.338 (0.200) & 0.245 (0.158)    	\\
				&			&		&	DSKM		&  0.077 (0.057) &  0.106 (0.107) & 0.157 (0.100)  	\\
				&			&		&	DDW		&  0.483 (0.055) &  0.490 (0.058) & 0.495 (0.090)   	\\
				&			&		&	DDKM		&  0.768 (0.217) &  0.768 (0.221) & 0.497 (0.130)  	\\ \cline{3-7}
				&					&\multirow{8}{*}{\begin{rotatebox}{90}{(1/3,1/3,1/3)}\end{rotatebox}}
									&	Ward		& 0.918 (0.075) & 0.957 (0.037) & 0.863 (0.105)	\\
				&			&		&	$K$-means	& 0.435 (0.071) & 0.485 (0.070) & 0.498 (0.077) 	\\
				&			&		&	SK-means		& 0.417 (0.089) & 0.480 (0.089) & 0.485 (0.099)  	\\
				&			&		&	MDP			& 0.361 (0.154) & 0.447 (0.095) & 0.347 (0.172)  	\\
				&			&		&	DSW			& 0.656 (0.131) & 0.696 (0.126) & 0.579 (0.099)   	\\
				&			&		&	DSKM		& 0.510 (0.081) & 0.610 (0.107) & 0.635 (0.089)  	\\
				&			&		&	DDW		& 0.930 (0.087) & 0.977 (0.041) & 0.919 (0.087)  	\\
				&			&		&	DDKM		& 0.586 (0.066) & 0.586 (0.072) & 0.884 (0.124) 	\\ 
\cline{2-7}
				& \multirow{16}{*}{2000}	&\multirow{8}{*}{\begin{rotatebox}{90}{(0.7,0.2,0.1)}\end{rotatebox}}
									&	Ward		&  0.488 (0.079) &  0.501 (0.073) &  0.136 (0.121) 	\\
				&			&		&	$K$-means	&  0.018 (0.053) & -0.019 (0.072) &  0.031 (0.036)  	\\
				&			&		&	SK-means		&  0.031 (0.036) &  0.009 (0.049) &  0.035 (0.033)   	\\
				&			&		&	MDP			& -0.116 (0.046) & -0.119 (0.018) &  0.056 (0.067)   	\\
				&			&		&	DSW			&  0.146 (0.130) &  0.248 (0.175) &  0.141 (0.109)    	\\
				&			&		&	DSKM		&  0.066 (0.046) &  0.064 (0.056) &  0.099 (0.076)  	\\
				&			&		&	DDW		&  0.484 (0.067) &  0.484 (0.053) &  0.503 (0.085)   	\\
				&			&		&	DDKM		&  0.894 (0.121) &  0.871 (0.149) &  0.875 (0.136)  	\\ \cline{3-7}
				&					&\multirow{8}{*}{\begin{rotatebox}{90}{(1/3,1/3,1/3)}\end{rotatebox}}
									&	Ward		& 0.529 (0.154) & 0.627 (0.222) & 0.685 (0.136)	\\
				&			&		&	$K$-means	& 0.288 (0.095) & 0.396 (0.103) & 0.373 (0.086) 	\\
				&			&		&	SK-means		& 0.210 (0.100) & 0.289 (0.127) & 0.276 (0.110)  	\\
				&			&		&	MDP			& 0.120 (0.100) & 0.171 (0.170) & 0.246 (0.166)  	\\
				&			&		&	DSW			& 0.442 (0.151) & 0.682 (0.140) & 0.377 (0.087)   	\\
				&			&		&	DSKM		& 0.356 (0.091) & 0.538 (0.116) & 0.417 (0.185)  	\\
				&			&		&	DDW		& 0.722 (0.160) & 0.747 (0.180) & 0.751 (0.162)  	\\
				&			&		&	DDKM		& 0.569 (0.048) & 0.577 (0.047) & 0.633 (0.137) 	\\ 
\hline
\end{tabular}
\end{table}

Consequently, 
the distance vector clustering approach show competitive performance in these numerical experiments.
%
\section{Application to Microarray Data}\label{section:5}%
%
Here, we apply the distance vector clustering approach to the preprocessed microarray gene expression datasets,
which are used in Dettling (2004).
These datasets are available at \url{http://stat.ethz.ch/~dettling/bagboost.html}.
For details about these datasets, see Dettling (2004).
As with Section $\ref{section:4}$, to compare the proposed methods with other clustering methods, 
we again chose Ward's method, $k$-means clustering, SK-means clustering, and MDP clustering as competitors.
We used the hierarchical type (Ward's and the single linkage methods) and $k$-means type distance vector clustering methods with the inner product matrix $S$ and the distance matrix $D$ 
as the proposed approach. 
Here, we refer to the $k$-means type distance vector clustering method using $S$ and $D$ as DSKM and DDKM, respectively.
Similarly, we also refer to the Ward (the single linkage) type distance vector clustering method using $S$ and $D$ as DSW (DSS) and DDW (DDS), respectively.
According to the suggestion in Ahn et al. (2013), 
we set $T=2$ and $G=5$ as the tuning parameters of the MDP clustering method.
Note that we fixed the number of clusters for each algorithm to make a straightforward comparison.

Table $\ref{table:1}$ shows the number of errors  for each method and each dataset.
From this table,
it appears that the results of the distance vector clustering approach are also competitive among the compared methods for real data.
\begin{table}[!h]
\caption{Clustering results of four preprocessed microarray gene expression datasets showing the number of errors for each dataset and each clustering method.}
\begin{tabular}{c|c|c|c|cccccccc}
\hline
Data 	&$N$& $p$	&$K$ &	Ward	&	$K$-means	&	SK-means		&	MDP 	 \\
\hline\hline
Colon	& 62	& 2000	& 2	 &	30		&	30			&	 10			&	30		\\
Leukemia	& 72	& 3571	& 2	 &	6		&	2			&	 2			&	36		\\
Lymphoma& 62& 4026	& 3	 &	1		&	1			&	 1			&	0		\\
Prostate	&102& 6033	& 2	 &	44		&	43			&	 41			&	42		\\
\hline
Data 	&$N$& $p$	&$K$ &DSW (DSS)	&	DSKM		&	DDW (DDS)	&	DDKM 	 \\
\hline\hline
Colon	& 62	& 2000	& 2	 &	31 (26)	&	30			&	 20 (24)		&	17		\\
Leukemia	& 72	& 3571	& 2	 &	4 (1)		&	3			&	 1 (26)		&	1		\\
Lymphoma& 62& 4026	& 3	 &	2 (11)	&	1			&	 2 (22)		&	1		\\
Prostate	&102& 6033	& 2	 &	44 (39)	&	43			&	 44 (45)		&	40		\\
\hline
\end{tabular}
\label{table:1}
\end{table}

We mention here that there are some differences between our table and the results in Ahn et al. (2013).
For example, in Ahn et al. (2013), the number of errors for the MDP clustering method for Colon data is $15$.
However, in this work,  the number of errors for the MDP clustering method for Colon data is $30$.
These differences may be because of differences in data preprocessing.
For the preprocessed data used in this study,  
the MDP distance for the split induced by the largest gap of the discarded first eigenvector is $6.970$ while
the MDP distance with the discarded second eigenvector is $6.551$. 
According to the algorithm of MDP clustering,
we choose the split which has the largest MDP distance.
Thus, we must choose the split induced by the first eigenvector, while the number of errors of the second eigenvector is $15$.


%
%
\section{Conclusion}\label{section:6}%
%
In this study, 
we pointed out the important fact that it is not the closeness, 
but the ``{\it values}" of distance that contain information of the cluster structure in high-dimensional space.
We proposed an efficient and simple clustering approach, called distance vector clustering, for HDLSS data based on that fact.
Under the assumption of Hall et al. (2005),
we showed that the proposed approach provides the true cluster label under milder conditions
when the dimension tends to infinity with the sample size fixed.
The effectiveness of the distance vector clustering approach was illustrated through numerical experiments and real data analysis.
Under some regularity conditions, in HDLSS data, 
we can detect the cluster structure, which consists of differences not only between mean vectors but also variances.
Moreover, we also showed that the distance vector clustering approach using the inner product matrix is less susceptible to the variances of hidden clusters than in the MDP clustering method.
Only the distance or the inner product matrix  and the usual clustering algorithm are needed for the distance vector clustering approach.
Thus, the distance vector clustering approach is easily implementable and understandable.
It can be considered from the present results that this approach is another possible choice for clustering HDLSS data.

In future work, 
we intend to provide an efficient selection method for the determination of the number of clusters by this method.
%
\section*{Acknowledgements}
The author wishes to express his thanks to Dr. Shota Katayama for his helpful discussions.
This work was supported by Grant-in-Aid for JSPS Fellows Number $24\cdot2466$.


\begin{thebibliography}{9}


\bibitem{r1}
\textsc{Ahn, J., Lee, M. H., and Yoon, Y. J.} (2013).
Clustering high dimension, low sample size data using the maximal data piling distance.
\textit{Statist. Sinica}.
\textbf{22} 443--464. 

\bibitem{r2}
\textsc{Ahn, J., Marron, J. S., Muller, K. M., and Chi, Y.-Y.} (2007).
The high-dimension, low-sample-size geometric representation holds under mild condition.
\textit{Biometrika}.
 \textbf{94} 760--766.
 
 \bibitem{r2}
\textsc{Borysov, P., Hannig, J., and Marron, J. S. } (2013).
Asymptotics of hierarchical clustering for growing dimension.
\textit{to appear in J. Multivar. Anal}.

\bibitem{r3}
\textsc{Dettling, M.} (2004).
BagBoosting for tumor classification with gene expression data.
\textit{Bioinformatics}.
\textbf{20} 3583--3593.
 
 \bibitem{r4}
\textsc{Hall, P., Marron, J. S., and Neeman, A.} (2005).
Geometric representation of high dimension, low sample size data.
\textit{J. R. Statist. Soc. B}.
 \textbf{67} 427--444.
 
 \bibitem{r5}
\textsc{Liu, L., Hayes, D. N., Nobel, A., and Marron. J. S.} (2008).
Statistical significance of clustering for high-dimensional, low-sample size data.
\textit{J. Amer. Statist. Assoc}.
 \textbf{58} 236--244.

 
  \bibitem{r6}
\textsc{Witten, D. M. and Tibshirani, R.} (2010).
A framework for feature selection in clustering.
\textit{J. Amer. Statist. Assoc}.
 \textbf{105} 713--726.




\end{thebibliography}
\end{document}